\documentclass[11pt]{amsart}

\usepackage{a4wide, amsmath, amsfonts, amssymb, mathrsfs, amsthm, breqn}
\usepackage[hyperfootnotes=false]{hyperref}

\allowdisplaybreaks[2]

\numberwithin{equation}{section}

\newtheorem{theorem}{Theorem}[section]

\newtheorem{corollary}[theorem]{Corollary}
\newtheorem{remark}[theorem]{Remark}
\newtheorem{definition}[theorem]{Definition}
\newtheorem{proposition}[theorem]{Proposition}
\newtheorem{example}[theorem]{Example}

\renewcommand{\d}{\mathrm{d}}
\renewcommand{\epsilon}{\varepsilon}
\renewcommand{\phi}{\varphi}

\newcommand{\dd}{\,\mathrm{d}}

\newcommand{\R}{\mathbb{R}}
\newcommand{\N}{\mathbb{N}}

\newcommand{\cN}{\mathcal{N\mkern-2mu N}}

\title[Distributionally robust approximation property of neural networks]{Distributionally robust approximation property of neural networks}

\author[Ceylan]{Mihriban Ceylan}
\address{Mihriban Ceylan, University of Mannheim, Germany}
\email{mihriban.ceylan@uni-mannheim.de}

\author[Pr{\"o}mel]{David~J. Pr{\"o}mel}
\address{David J. Pr{\"o}mel, University of Mannheim, Germany}
\email{proemel@uni-mannheim.de}

\date{\today}

\begin{document}

\begin{abstract}
  The universal approximation property uniformly with respect to weakly compact families of measures is established for several classes of neural networks. To that end, we prove that these neural networks are dense in Orlicz spaces, thereby extending classical universal approximation theorems even beyond the traditional $L^p$-setting. The covered classes of neural networks include widely used architectures like feedforward neural networks with non-polynomial activation functions, deep narrow networks with ReLU activation functions and functional input neural networks.
\end{abstract}

\maketitle






\section{Introduction}

The remarkable success of neural networks across modern machine learning is often traced back to their universal approximation property: sufficiently large neural networks can approximate any function within a given suitable function space. This fundamental feature provides a theoretical justification why neural networks are used at the heart of many machine learning algorithms, successfully applied in various areas like image classification \cite{Krizhevsky2012}, speech recognition \cite{Hinton2012} and game playing \cite{Silver2016}. Since the pioneering works \cite{McCulloch1943,Rosenblatt1958}, the study of neural networks has been driven by both applied and theoretical challenges, arising in diverse research fields ranging from approximation theory and optimization to statistical learning theory; see \cite{Bengio2017,Aggarwal2023} for introductory textbooks.

As established in the classical universal approximation theorems (UATs)
\cite{Cybenko1989,Funahashi1989,Hornik1989}, feedforward neural networks with continuous sigmoidal activation functions are dense in the space of continuous functions on compact subsets of the Euclidean space. Subsequent works extended this universal approximation property to broader classes of neural networks allowing for, e.g., non-polynomial activation functions, see \cite{Hornik1991,Mhaskar1992,Leshno1993,Chen1995,Pinkus1999}. Beyond the space of continuous functions, universal approximation theorems have been proven for various function spaces, with the most prominent examples being Sobolev spaces \cite{Hornik1990,Hornik1991} and $L^p$-spaces \cite{Cybenko1989,Hornik1989,Mhaskar1992,Leshno1993}; for more recent contributions see also, e.g., \cite{Kidger2020,Guhring2020,Neufeld2024}. However, note that the universal approximation theorems in $L^p(\mu)$-spaces are formulated with respect to a fixed single measure or distribution~$\mu$.

In many practical settings, the assumption that the underlying (data) distribution is known precisely is unrealistic. Uncertainty or ambiguity about this distribution---often referred to as distributional uncertainty---is an omnipresent challenge across the sciences. In decision theory, economics and finance, it is known as Knightian uncertainty \cite{Marinacci2015}, while in statistics and optimization it is studied under the frameworks of robust statistics \cite{Huber2009} and distributionally robust optimization \cite{Kuhn2025}, which often formalize such ambiguity through weakly compact sets of probability measures or Wasserstein neighborhoods \cite{BenTal2013,GaoKleywegt2016,BlanchetMurthy2019}. From a mathematical perspective, 'robustness' corresponds to the ability of a method or approximation to perform well under uncertainty, that is, uniform across a family of measures and not only with respect to a single distribution. This notion of robustness is central to machine learning and data science, for instance, in the context of adversarial training \cite{Bai2023}, overparameterized neural networks \cite{Sagawa2019} and noisy data \cite{Kratsios2025}.

The aim of the present work is to establish the \emph{distributionally robust approximation property} of neural networks in the following sense: given a weakly compact set $\mathcal M$ of finite Borel measures on $\R^{N_0}$ and a suitable function $f\colon \R^{N_{0}}\to \R^{N_L}$, for every $\epsilon>0$ there exists a neural network $\eta$ such that
\begin{equation*}
  \sup_{\nu\in\mathcal M}\|f-\eta\|_{L^1(\nu)}< \epsilon.
\end{equation*}
It turns out that this distributionally robust approximation property can be derived from suitably formulated universal approximation theorems in Orlicz spaces. Proving these theorems within the general framework of Orlicz spaces constitutes the main mathematical challenge.

As first contribution, we prove \emph{universal approximation theorems in Orlicz spaces} for several classes of neural networks, showing that these networks are dense with respect to the Luxemburg norm; see Theorem~\ref{thm: UAT in Orlicz spaces} for a detailed summary of the results. The covered architectures include widely used models such as feedforward neural networks with non-polynomial activation functions, deep narrow networks with ReLU activation, and functional-input neural networks, where the precise architecture depends on the underlying reference measure. Since Orlicz spaces naturally generalize the notion of $L^p$-spaces, these results, in particular, extend classical universal approximation theorems known for $L^p$-spaces, cf. \cite{Hornik1991,Leshno1993,Kidger2020}. Beyond their theoretical generality, UATs in Orlicz spaces are of independent interest since, for instance, loss functions such as cross-entropy and the Kullback--Leibler divergence correspond to exponential Orlicz spaces~\cite{Cover2006} and partial differential equations with non-standard growth conditions are canonically formulated and analyzed in Orlicz type spaces~\cite{Chlebicka2018}. Moreover, recent advances in approximation theory have investigated the universality of neural network operators within frameworks of Orlicz spaces~\cite{Costarelli2019,Costarelli2025}.

As a second contribution, we derive a \emph{robust universal approximation theorem} for the aforementioned classes of neural networks with respect to weakly compact families of measures, showing that these neural networks possess the distributionally robust approximation property; see Theorem~\ref{thm: Robust UAT} for the precise statement. While this robust version extends various classical universal approximation theorems to a uniform setting over families of measures, it also provides a rigorous theoretical foundation for applications in, e.g., distributionally robust learning and statistical learning, where robustness against distributional shifts or ambiguity in the underlying data-generating process is essential.

\medskip
\noindent\textbf{Acknowledgments:} The universal approximation theorems in Orlicz spaces for various classes of neural networks are established in Section~\ref{sec: UAT in Orlicz spaces}. The distributionally robust approximation property for these classes of neural networks is derived in Section~\ref{sec: Robust UAT}. Section~\ref{sec: infinite} discusses extensions to infinite-dimensional input and output spaces. Appendix~\ref{sec:appendix} presents the essential foundation of Orlicz spaces.

\medskip
\noindent\textbf{Acknowledgments:} M. Ceylan gratefully acknowledges financial support by the doctoral scholarship programme from the Avicenna-Studienwerk, Germany.

\section{Universal Approximation Theorems in Orlicz Spaces}~\label{sec: UAT in Orlicz spaces}
\noindent In this section, we derive universal approximation theorems in Orlicz spaces for various classes of neural networks, allowing for bounded, non-polynomial, and ReLU activation functions, as well as for functional input neural networks, thereby ensuring that the presented approximation results hold for neural networks with a broad class of activation functions. Before presenting these findings, we recall the basic definitions regarding neural networks.

\subsection{Essentials on Neural Networks}\label{subsec: essentials on NN}

The space $\R^n$ is equipped with the Euclidean norm $\|\cdot\|$ and $\langle\cdot,\cdot\rangle_{\R^n}$ denotes the usual Euclidean inner product. On $\R^n$, we work with the Borel $\sigma$-algebra $\mathcal{B}(\R^{n})$. We denote by $C(\R^n;\R^m)$ the space of all continuous functions $f\colon\R^n\to\R^m$. Moreover, we write $C_b(\R^n;\R^m)$ for the space of all continuous and bounded functions $f\colon\R^n\to\R^m$ and $C_c(\R^n;\R^m)$ for the space of continuous and compactly supported functions $f\colon\R^n\to\R^m$.

\medskip

Let $L, N_0, \ldots, N_L \in \N$ and, for any $l \in \{1, \ldots, L\}$, let $w_l \colon \R^{N_{l-1}} \to \R^{N_l}$, $x \mapsto A_l x + b_l$, be an affine function with $A_l \in \R^{N_l \times N_{l-1}}$ and $b_l \in \R^{N_l}$. Given an activation function $\varrho\colon\R\to\R$, a  \emph{deep feedforward neural network} $\eta\colon \R^{N_0} \to \R^{N_L}$ is defined by
\begin{equation*}
  \eta = w_L \circ \varrho \circ w_{L-1} \circ \ldots \circ \varrho \circ w_1,
\end{equation*}
where $\circ$ denotes the usual composition of functions. Here, $\varrho$ is applied componentwise, $L-1$ denotes the number of hidden layers ($L$ is the depth of $\eta$), and $N_1, \ldots, N_{L-1}$ denote the dimensions (widths) of the hidden layers and $N_0$ and $N_L$ the dimension of the input and the output layer, respectively. We write $\cN^\varrho_{N_0,N_L,L,k}$ for the set of all feedforward neural networks $\eta \colon \R^{N_0} \to \R^{N_L}$ with activation function $\varrho$, input dimension $N_0$, output dimension~$N_L$ and a number of hidden layers $L-1$ with $k$ hidden units.

If the number of hidden units is arbitrary, we write $\cN^\varrho_{N_0,N_L,L,\infty}$ and if the number of hidden layers $(L-1)$ is arbitrary, we similarly write $\cN^\varrho_{N_0,N_L,\infty,k}$. For more details on feedforward networks we refer, e.g., to \cite{Mukherjee2018,Daubechies2022,Aggarwal2023}.

\begin{example}
  A frequently used example is that of neural networks with a single hidden layer, see e.g., \cite{Hornik1991}. Given an activation function $\varrho\colon \R\to\R$, a neural network $\eta\colon \R^m\to\R$ with a single hidden layer of width $n$ is given by
  \begin{equation*}
    \eta(x)=\sum_{j=1}^{n}w_j\varrho(a_j^\top x+b_j),\quad x\in \R^m,
  \end{equation*}
  where $w_1,\ldots,w_n\in\R$, $a_1,\ldots, a_n\in\R^m$ and $b_1,\ldots,b_n\in \R$ represent the linear readouts, weight vectors and biases, respectively.
\end{example}

To define a broader class of architectures, we consider so-called \emph{functional input neural networks}. To that end, we follow and adapt the framework of~\cite{Cuchiero2024} to the present setting, allowing us to derive an approximation result in Orlicz spaces below.  We keep the terminology from~\cite{Cuchiero2024}, where this class is introduced on general weighted input spaces, which may in particular be infinite-dimensional. In the present paper, we focus on the finite-dimensional specialization with input space $\R^{N_0}$.
\medskip

A function $\psi\colon \R^{N_0}\to (0,\infty)$ is called an \emph{admissible weight function} if every pre-image $K_R:=\psi^{-1} ((0,R])=\{x\in \R^{N_0}: \psi(x)\le R\}$ is compact for all $R>0$. Then, we call the pair $(\R^{N_0},\psi)$ a \emph{weighted space}. We define the function space
\begin{equation*}
  B_\psi(\R^{N_0};\R^{N_2}):=\Bigg\{f\colon \R^{N_0}\to \R^{N_2}: \sup_{x\in \R^{N_0}}\frac{\|f(x)\|}{\psi(x)}<\infty\Bigg\},
\end{equation*}
that is, the space of maps $f\colon \R^{N_0}\to \R^{N_2}$, whose growth is controlled by the growth of the weight function $\psi\colon \R^{N_0}\to (0,\infty)$. The vector space $B_\psi(\R^{N_0};\R^{N_2})$ is equipped with the weighted norm $\|\cdot\|_{\mathcal B_\psi(\R^{N_0};\R^{N_2})}$ given by
\begin{equation}\label{eq:weightednorm}
  \|f\|_{\mathcal B_\psi(\R^{N_0};\R^{N_2})}=\sup_{x\in \R^{N_0}}\frac{\|f(x)\|}{\psi(x)}, \quad \text{for }f\in B_\psi(\R^{N_0};\R^{N_2}).
\end{equation}
Note that the embedding $C_b(\R^{N_0};\R^{N_2})\hookrightarrow B_\psi(\R^{N_0};\R^{N_2})$ is continuous, and we set
\begin{equation*}
  \mathcal{B}_\psi(\R^{N_0};\R^{N_2})
  := \overline{C_b(\R^{N_0};\R^{N_2})}^{\,\|\cdot\|_{\mathcal{B}_\psi(\R^{N_0};\R^{N_2})}},
  \qquad
  \mathcal{B}_\psi(\R^{N_0}):=\mathcal{B}_\psi(\R^{N_0};\R),
\end{equation*}
which is a Banach space with norm \eqref{eq:weightednorm}.

\medskip

Furthermore, to specify the admissible hidden features of the considered neural networks, we introduce the notion of an
\emph{additive family}. 
We call a subset $\mathcal H\subseteq \mathcal B_\psi(\R^{N_0})$ an \emph{additive family} (on $\R^{N_0}$) if
\begin{enumerate}
  \item[(i)] $\mathcal H$ is closed under addition, i.e.,~for every $h_1,h_2\in\mathcal H$ it holds that $h_1+h_2\in\mathcal H$,
  \item[(ii)] $\mathcal H$ is point separating, i.e., for distinct $x_1,x_2\in \R^{N_0}$ there is $h\in\mathcal H$ with $h(x_1)\neq h(x_2)$,
  \item[(iii)] $\mathcal H$ contains the constant functions, i.e., for every $b\in\R$ we have $(x\mapsto b)\in\mathcal H$.
\end{enumerate}
  
We assume that the hidden layer space $(\R,\psi_1)$ is a weighted space. In this setting, we call a function $\varrho\in C(\R;\R)$ $\psi_1$\emph{-activating} if, for a given admissible weight function $\psi_1\colon\R\to (0,\infty)$, the set of neural networks
\begin{equation*}
  \cN_{\R,\R}^\varrho=\Bigl\{\R\ni z\mapsto \sum_{n=1}^N w_n\varrho(a_nz+b_n)\in\R:N\in\N,a_n\in\N_0,b_n\in\R\Bigr\}
\end{equation*}
is a dense subset of $\mathcal B_{\psi_1}(\R)$.

\medskip

With these preliminaries in place, we are able to introduce so-called \emph{functional input neural networks}, based on the weighted input space $(\R^{N_0},\psi)$ and the weighted hidden space $(\R,\psi_1)$ introduced above. For a given additive family $\mathcal H\subseteq\mathcal B_\psi(\R^{N_0})$, a function $\varrho\in C(\R;\R)$, and a vector subspace $\mathcal L\subseteq \R^{N_2}$, we define a \emph{functional input neural network (FNN)} $\eta\colon \R^{N_0}\to \R^{N_2}$ as functions of the form
\begin{equation}\label{eq: FNN}
  \eta(x)=\sum_{n=1}^Ny_n\varrho(h_n(x)),\quad \text{for } x\in \R^{N_0},
\end{equation}
where $N\in\N$ denotes the number of neurons, $h_1,\ldots,h_N\in\mathcal H$ are the hidden layer maps, and $y_1,\ldots,y_N\in\mathcal L$ represent the linear readouts. Moreover, we denote by $\cN_{\R^{N_0},\R^{N_2}}^{\mathcal H,\varrho,\mathcal L}$ the set of FNNs of the form \eqref{eq: FNN}. In what follows, we consider $\mathcal L=\R^{N_2}$ and write $\cN_{\R^{N_0},\R^{N_2}}^{\mathcal H,\varrho}$ for the resulting networks.
  
\medskip

In particular, taking the affine family
\begin{equation}\label{eq: additive family}
  \widetilde{\mathcal H}:=\Bigl\{x\mapsto a^\top x+b: a\in\R^{N_0},b\in \R\Bigr\},
\end{equation}
recovers the classical notion of feedforward neural networks with single layer, see \cite[Remark~4.7]{Cuchiero2024}. In this case we write $\cN_{\R^{N_0},\R^{N_2}}^{\widetilde{\mathcal H},\varrho}$.

\subsection{Main Universal Approximation Theorem and Orlicz Spaces}
 
Before stating the main theorem, a universal approximation theorem for Orlicz spaces, we briefly recall the necessary definition regarding Orlicz spaces. For more essential background on Orlicz spaces we refer to Appendix~\ref{sec:appendix}.

\medskip

Let $\phi$ be a Young function, as defined in Definition~\ref{def: Young function}. The \emph{Orlicz space} is defined as
\begin{equation*}
  \mathcal L^\phi(\mu;\R^{N_L}):=\bigg\{f\colon\R^{N_0}\to\R^{N_L}\,:\, \int_{\R^{N_0}} \phi(\alpha \|f\|)\dd\mu<\infty\text{ for some } \alpha>0\bigg\},
\end{equation*}
and the \emph{Orlicz heart} as
\begin{equation*}
  M^\phi(\mu;\R^{N_L}):=\bigg\{f\colon\R^{N_0}\to\R^{N_L}\,:\,\int_{\R^{N_0}} \phi(\alpha \|f\|)\dd\mu<\infty\text{ for all } \alpha>0\bigg\},
\end{equation*}
which we equip both with the gauge norm
\begin{equation*}
  N_{\phi,\mu}(f):=\inf\bigg\{k>0:\,\int_{\R^{N_0}}\phi\Bigl(\frac{\|f\|}{k}\Bigr)\dd\mu\le 1\bigg\}.
\end{equation*}
Note that $M^\phi(\mu;\R^{N_L})$ is a linear subspace of $\mathcal{L}^\phi(\mu;\R^{N_L})$. Some conditions ensuring the identity $\mathcal{L}^\phi(\mu;\R^{N_L}) =  M^\phi(\mu;\R^{N_L})$ can be found in Remark~\ref{rem: M phi}~(i). Moreover, we refer to Definition~\ref{def: N-function} for the definition of $N$-functions. Whenever the dimension of the image space is clear from context, we simply write $M^\phi(\mu)$ and $L^\phi(\mu)$.

\begin{remark}\label{rem: Lebesgue spaces}
  Taking $\phi(x):=|x|^p$ for $x\in\R$ and for fixed $p\in [1,\infty)$, we have 
  \begin{equation*}
  \mathcal{L}^\phi(\mu;\R^{N_L}) =  M^\phi(\mu;\R^{N_L}) = L^p(\mu;\R^{N_L}),
  \end{equation*} that is, the Orlicz space $\mathcal{L}^\phi(\mu;\R^{N_L})$ coincides with the classical Lebesgue space $L^p(\mu;\R^{N_L})$. Recall that, for a given finite Borel measure $\mu$ on the measurable space $(\R^{N_0},\mathcal{B}(\R^{N_0}))$, the space $L^p(\mu)=L^p(\R^{N_0},\mathcal{B}(\R^{N_0}),\mu;\R^{N_L})$ consists of all measurable functions $f\colon \R^{N_0} \to \R^{N_L}$ such that
  \begin{equation*}
    \int_{\R^{N_0}} \|f(x)\|^p\dd\mu(x) < \infty.
  \end{equation*}
\end{remark}

With these preliminaries in place, we can summarize in the next theorem the universal approximation property of various classes of neural networks in Orlicz spaces. Note that the specific architecture of the involved neural networks is explicitly stated in the corresponding propositions, which are presented in the subsequent subsections.

\begin{theorem}[Universal Approximation Theorem in Orlicz Spaces]\label{thm: UAT in Orlicz spaces}
  Let $\phi$ be an $N$-function and $\mu$ be a locally finite Borel measure on $(\R^{N_0},\mathcal B(\R^{N_0}))$. Then, the set of neural networks is dense in the Orlicz heart $M^\phi(\mu)$ with respect to the gauge norm $N_{\phi,\mu}$ in the following cases, where the corresponding network architecture is specified separately:
  \begin{enumerate}
    \item[(i)] If $\mu$ is finite and $\varrho$ is measurable, bounded and non-constant, then, for every  $L\geq 2$, $\eta\in \cN^\varrho_{N_0,N_L,L,\infty}. $ This is the content of Proposition~\ref{prop: UAT bounded};

    \item[(ii)] If $\varrho$ is the ReLU activation function, then $ \eta\in \cN^\varrho_{N_0,N_L,\infty,N_0+N_L+1}$. This is the content of Proposition~\ref{prop: Orlicz ReLU};

    \item[(iii)] If $\varrho\in C(\R;\R)$ is non-polynomial and $\operatorname{supp}\mu\subset K$ for some compact set $K\subset\R^{N_0}$, then, for every $L\geq 2$, $\eta\in \cN^\varrho_{N_0,N_L,L,\infty}$. This is the content of Proposition~\ref{prop: UAT Orlicz nonpolynomial};

    \item[(iv)] If $\mu$ is finite, $\mathcal H\subseteq\mathcal B_\psi(\mathbb R^{N_0})$ is an additive family, $\varrho\in C(\R;\R)$ is $\psi_1$-activating,
    \begin{equation*}
      \sup_{x\in\R^{N_0}}\frac{\psi_1(h(x))}{\psi(x)}<\infty \quad\text{for all }h\in\mathcal H,
    \end{equation*}
   and $N_{\phi,\mu}(\psi)<\infty$, then $\eta\in \cN^{\mathcal H,\varrho}_{\R^{N_0},\R^{N_2}}$. This is the content of Proposition~\ref{prop: UAT functional}.
  \end{enumerate}
\end{theorem}

In the following subsections, we shall prove the corresponding propositions (Proposition~\ref{prop: UAT bounded}, \ref{prop: Orlicz ReLU}, \ref{prop: UAT Orlicz nonpolynomial}, and~\ref{prop: UAT functional}) and discuss them in more detail.

\subsection{Neural Networks with Bounded Activation Function}

In this subsection, we derive a universal approximation theorem in Orlicz spaces, allowing for neural networks with bounded activation functions, for which the number of hidden layers can be any fixed number, and allowing for general finite Borel measures on the Euclidean space $\R^{N_0}$.

\begin{proposition}\label{prop: UAT bounded}
  Let $\phi$ be an $N$-function, $\varrho$ be a measurable, bounded and non-constant activation function, and $L\geq 2$. Then, $\cN^\varrho_{N_0,N_L,L,\infty}$ is dense in $M^\phi(\mu)$ with respect to the gauge norm $N_{\phi,\mu}$ for every finite Borel measure $\mu$ on $(\R^{N_0},\mathcal{B}(\R^{N_0}))$.
\end{proposition}

\begin{proof}
  First, note that since $\varrho$ is bounded and $\mu$ is a finite measure on $\R^{N_0}$, by Remark \ref{rem: M phi}~(iii) $\cN^\varrho_{N_0,N_L,L,\infty}$ is a linear subspace of $M^\phi(\mu)$.

  Now, if for some $\mu$,\, $\cN^\varrho_{N_0,N_L,L,\infty}$ is not dense in $M^\phi(\mu)$, then the Hahn--Banach Theorem see e.g., \cite[Corollary~4.8.7]{Friedman1982} yields that there is a nonzero continuous linear functional~$F$ on $M^\phi(\mu)$ that vanishes on $\cN^\varrho_{N_0,N_L,L,\infty}$. According to Theorem~\ref{thm: Riesz for Orlicz}, there exists a unique function $g\in L^\psi(\mu)$, where $\psi$ is the complementary function of the $N$-function $\phi$, such that $F\colon M^\phi(\mu) \to \R$ is of the form
  \begin{equation*}
    F(f)=\int_{\R^{N_0}}\langle f(x), g(x)\rangle_{\R^{N_L}}\dd\mu(x),
  \end{equation*}
  for all $f\in M^\phi(\mu).$

  Define a vector measure $\sigma\colon\mathcal{B}(\R^{N_0})\to\R^{N_L}$ by
  \begin{equation*}
    \sigma(B):=\int_Bg\dd\mu, \quad B\in \mathcal{B}(\R^{N_0}).
  \end{equation*}
  Then, we have by the generalized H{\"o}lder's inequality (Proposition~\ref{prop: genHölder}), that for all $B\in \mathcal{B}(\R^{N_0})$, (where without loss of generality $\mu(B)\neq 0$ since, for $\mu(B)=0$, we have $N_{\phi,\mu}(\mathbf 1_B)=0$),
  \begin{equation*}
    \|\sigma(B)\|=\bigg\|\int_{\R^{N_0}}\mathbf 1_Bg\dd\mu\bigg\|\le 2 N_{\phi,\mu}(\mathbf 1_B)N_{\psi,\mu}(g)< \infty,
  \end{equation*}
  where we used that bounded functions are in $M^\phi(\mu)$, see Remark~\ref{rem: M phi}~(iii).

  Hence, $\sigma=(\sigma_1,\ldots,\sigma_{N_L})$ is a nonzero finite signed vector measure on $\R^{N_0}$ such that
  \begin{equation*}
    F(f)=\int_{\R^{N_0}}\langle f(x),g(x)\rangle_{\R^{N_L}}\dd\mu(x)=\int_{\R^{N_0}}\langle f(x),\dd\sigma(x)\rangle_{\R^{N_L}}, \quad f \in  M^\phi(\mu).
  \end{equation*}
   As $F$ vanishes on $\cN^\varrho_{N_0,N_L,L,\infty}$, it in particular vanishes on the following subclass of networks. Since $\varrho$ is non-constant, choose $u,v\in\R$ such that $\varrho(u)\neq \varrho(v).$ We first construct a bounded non-constant scalar function which can be realized along one path through the hidden layers. If $L=2$, set  $\theta_{L-1}:=\theta_1:=\varrho$. If $L>2$, we define $\theta_2,\ldots,\theta_{L-1}$ recursively as follows. Assume that $\theta_k$ has been constructed and is non-constant. Then choose $s,t\in\R$ such that $\theta_k(s)\neq\theta_k(t)$. Hence, there exists an affine map $T_k\colon \R\to\R$ such that $T_k(\theta_k(s))=u$ and $T_k(\theta_k(t))=v$. We then define $\theta_{k+1}:=\varrho\circ T_k\circ\theta_k$. It follows that $\theta_{k+1}(s)=\varrho(u)\neq \varrho(v)=\theta_{k+1}(t)$, and hence $\theta_{k+1}$ is again non-constant. Iterating  this construction up to $k=L-2$, we obtain
  \begin{equation*}
    \theta_{L-1}=\varrho\circ T_{L-2}\circ\varrho\circ\cdots\circ T_1\circ \varrho,
  \end{equation*}
  which is bounded and non-constant.

  Now let $a\in\R^{N_0}$ and $b\in\R$ be arbitrary. Choose the first hidden layer such that the $j_1$-th component of $A_1x+b_1$ equals $a^\top x+b$. Writing $T_k(z)=\alpha_k z+\beta_k, k=1,\ldots,L-2,$ we choose, for $l=2,\ldots,L-1$,
  \begin{equation*}
    A_l=\alpha_{l-1}e_{j_l}e_{j_{l-1}}^\top\in\R^{N_l\times N_{l-1}}, \qquad b_l=\beta_{l-1}e_{j_l},
  \end{equation*}
  where $e_{j_l}$ denotes the $j_l$-th standard basis vector in $\R^{N_l}$. Now, fix $r\in\{1,\ldots,N_L\}$. We choose the output layer $A_L$ such that its $j_{L-1}$-th column equals $e_r\in\R^{N_L}$, all other columns are zero, and $b_L=0$. With this construction, the selected scalar component is propagated through the hidden layers according to the affine maps $T_k$ and is placed in the $r$-th output component. Hence the resulting network has the form $\eta(x)=e_r\theta_{L-1}(a^\top x+b).$ Therefore,
  \begin{equation*}
    0=F(\eta)
    =\int_{\R^{N_0}}\langle \eta(x),\dd\sigma(x)\rangle_{\R^{N_L}}
    =\int_{\R^{N_0}}\theta_{L-1}(a^\top x+b)\dd\sigma_r(x),
  \end{equation*}
  for all $a\in\R^{N_0}$ and $b\in\R$. Since $\theta$ is bounded and non-constant, it is discriminatory by  \cite{Cybenko1989} and \cite{Hornik1991}.

  This means that there exists no nonzero finite signed measure $\lambda$ on $\R^{N_0}$ such that
  \begin{equation*}
    \int_{\R^{N_0}}\theta_{L-1}(a^\top x+b)\dd\lambda(x)=0
    \qquad
    \text{for all }a\in\R^{N_0},\ b\in\R.
  \end{equation*}
  Applying this to the finite signed measure $\sigma_r$, the preceding identity implies that $\sigma_r=0$. Since $r\in\{1,\ldots,N_L\}$ was arbitrary, we obtain
  \begin{equation*}
    \sigma_1=\cdots=\sigma_{N_L}=0,
  \end{equation*}
  and therefore $\sigma=0$. However, this contradicts the assumption that $\sigma$ is nonzero, which completes the proof.
\end{proof}

\begin{remark}
  Proposition~\ref{prop: UAT bounded} generalizes the classical universal approximation theorems by \cite{Cybenko1989,Hornik1991} see \cite[Theorem~1]{Hornik1991}. Indeed, choosing $\phi(x):=|x|^p$, $x\in \R$, for $p\in [1,\infty)$, we immediately obtain, by Remark~\ref{rem: Lebesgue spaces}, the following $L^p$-universal approximation theorem:

  Let $1\le p<\infty$ and $\varrho\colon \R\to \R$ be a bounded and non-constant function. Then, for every finite Borel measure $\mu$ on $(\R^{N_0},\mathcal{B}(\R^{N_0}))$, the set $\cN^\varrho_{N_0,N_L,L,\infty}$ is dense in $L^p(\mu)$.
\end{remark}

\subsection{Neural Networks with ReLU Activation Function}

In this subsection we derive the universal approximation property of feedforward neural networks, generated by the widely used rectified linear unit (ReLU) activation function $\varrho(x):=\max\{x,0\}$, with bounded width and arbitrary depth, allowing for general locally finite Borel measures on the Euclidean space~$\R^{N_0}$.

\begin{proposition}\label{prop: Orlicz ReLU}
  Let $\varrho$ be the ReLU activation function and let $\phi$ be an $N$-function. Then, $\cN_{N_0,N_L,\infty,N_0+N_{L}+1}^\varrho$ is dense in $M^\phi(\mu)$ with respect to the gauge norm $N_{\phi,\mu}$ for every locally finite Borel measure $\mu$ on $(\R^{N_0},\mathcal{B}(\R^{N_0}))$.
\end{proposition}

\begin{proof}
  For notational simplicity, we assume that $\R^{N_L}$ is endowed with the maximum norm $\|\cdot\|_\infty$ and, with some abuse of notation, we also denote the maximum norm by $\|\cdot\|$.
  
  Let $f\in M^\phi(\mu)$. Then, by Proposition~\ref{prop: compact dense in M phi} there exists a continuous function with compact support $h=(h_1,\ldots,h_{N_L})\in C_c(\R^{N_0};\R^{N_L})$ such that
  \begin{equation}\label{eq: 1}
    N_{\phi,\mu}(f-h)\le \frac\epsilon 2.
  \end{equation}
  We set
  \begin{equation*}
    C:=\sup_{x\in\R^{N_0}}\max_{1\le i \le {N_L}}h_i(x)+1
    \quad \text{and}\quad
    c:=\inf_{x\in\R^{N_0}}\min_{1\le i \le {N_L}}h_i(x)-1.
  \end{equation*}
  Let $a_1,b_1,\ldots,a_{N_0},b_{N_0}\in \R$ be such that $J$ defined by
  \begin{equation*}
    J:=[a_1,b_1]\times\cdots\times [a_{N_0},b_{N_0}]
  \end{equation*}
  is such that $\operatorname{supp} h \subseteq J$. Moreover, for $\delta>0$, let
  \begin{align*}
    A_i:=a_i-\delta,
    \quad \text{and}\quad
    B_i:=b_i+\delta,
  \end{align*}
  and define
  \begin{equation*}
    K:=[A_1,B_1]\times\cdots\times[A_{N_0},B_{N_0}].
  \end{equation*}
  Now, fix $\tilde\lambda=\frac \epsilon 2$ and choose $\delta$ small enough that
  \begin{equation}
    \mu(K\setminus J)\phi\Bigl(\frac{2\max\{|C|,|c|\}}{\tilde\lambda}\Bigr)<\frac 1 2.\label{eq: 3}
  \end{equation}

  Then, by Proposition~4.9 of \cite{Kidger2020} there is a $g=(g_1,\ldots,g_{N_L})\in \cN_{N_0,N_L,\infty,N_0+N_L+1}^\varrho$ such that
  \begin{equation}
    \sup_{x\in K}\|h(x)-g(x)\|<\delta^{\prime},\label{eq: 4}
  \end{equation}
  where $\delta^{\prime}$ will be chosen in a moment.

  In the following, we proceed in the spirit of the proof of \cite[Theorem~4.16]{Kidger2020} to modify $g$ and construct $G\in\cN_{N_0,N_L,\infty,N_0+N_L+1}^\varrho$, by removing the output layer and adding some more hidden layers (in total $3N_0+1$), such that it takes value $g$ on $J$ and zero on $\R^{N_0}\setminus K$.  In particular, we construct a neural representation of
  \begin{equation*}
    G_j=\min\{\max\{g_j,cV\},CV\},\quad j=1,\ldots,N_L,
  \end{equation*}
  where $V\colon \R^{N_0}\to[0,1]$ is constructed such that it satisfies
  \begin{equation*}
    V(x)=1~\text{for}\, x\in J,\quad V(x)=0~\text{for}\,x\notin K,\quad 0\le V(x)\le 1~\text{otherwise}.
  \end{equation*}

  For notational convenience, we introduce the notion of an enhanced neuron. By this we mean any map of the form $\ell_1\circ\varrho\circ \ell_2$, where $\ell_1,\ell_2$ are affine maps. Furthermore, it is allowed to take affine combinations of enhanced neurons and one may use that for $N>0$ sufficiently large, $x\mapsto\varrho(x+N)-N$ coincides with the identity function. Hence, one enhanced neuron may represent the identity function. This approach allows us to capture the inputs at every hidden layer (called in-register neurons) and preserve the values of those designated neurons through the layers.

  We now modify $g$ by first removing the output layer. Secondly, we want to obtain a mapping $V\colon\R^{N_0}\to[0,1]$. To that end, we use that two layers of two enhanced neurons each may represent the continuous piecewise affine function $V_i\colon \R\to\R$, where
  \begin{equation*}
    V_i(x)=1,\,x\in[a_i,b_i],~\text{and}~ V_i(x)=0,\,x\in(-\infty,a_i-\delta]\cup[b_i+\delta,\infty),\,i=1,\ldots,N_0,
  \end{equation*}
 see \cite[Lemma~B.1]{Kidger2020}. Hence, we are able to replace the value of the $x_i$-in-register neurons with $V_i(x_i)$, $i=1,\ldots,N_0$, through adding $2N_0$ hidden layers.

  Analogously, it follows from Lemma~B.2 of \cite{Kidger2020} that a single layer containing two enhanced neurons can implement the function $[0,\infty)^2\ni(x,y)\mapsto\min\{x,y\}$. By successively applying the above construction and adding $N_0-1$ hidden layers, we are able to compute
  \begin{equation*}
    V(x):=\min_{i=1,\ldots,N_0}V_i(x_i),
  \end{equation*}
  and to store its value in one of the in-register neurons. Consequently, the map $V\colon\R^{N_0}\to [0,1]$ is realized by adding $2N_0+(N_0-1)$ hidden layers in total. With that construction $V$ now approximates the indicator function $\mathbf{1}_J$, with support in $K$, taking value $0$ on $\R^{N_0}\setminus K$.

  Now to obtain $G\in\cN^\varrho_{N_0,N_L,\infty,N_0+N_L+1}$, we add two more hidden layers and the output layer, with the value of its neurons denoted by $G_1,\ldots,G_{N_L}$, which yields
  \begin{equation*}
    G_j:=-\varrho(-\varrho(g_j-cV)+(C-c)V)+CV=\min\{\max\{g_j,cV\},CV\},\quad j=1,\ldots,N_L,
  \end{equation*}
  where we used that
  \begin{equation*}
    \R^2\ni(x,y)\mapsto\max\{x,y\}=\varrho(x-y)+y\quad\text{and}\quad \R^2\ni(x,y)\mapsto\min\{x,y\}=x-\varrho(x-y).
  \end{equation*}
  Thus, as intended, this construction produces
  \begin{equation*}
    G_j=\min\{\max\{g_j,cV\},CV\},
  \end{equation*}
  and we deduce that
  \begin{equation*}
    G(x)=g(x)\quad\text{for}\quad x\in J,\quad\text{and}\quad G(x)=0\quad\text{for}\quad x\in\R^{N_0}\setminus K.
  \end{equation*}
  Therefore, since $G=(G_1,\ldots,G_{N_L})$ and $g=(g_1,\ldots,g_{N_L})$ coincide on $J$, and by \eqref{eq: 4} and the monotonicity of $\phi$, we have
  \begin{align*}
    \int_J\phi\Bigl(\frac{\|h(x)-G(x)\|}{\tilde\lambda}\Bigr)\dd\mu(x)&\le \mu(J)\sup_{x\in J}\phi\Bigl(\frac{\|h(x)-g(x)\|}{\tilde\lambda}\Bigr)\\
    &\le \mu(J)\phi\Bigl(\frac{\delta^{\prime}}{\tilde\lambda}\Bigr).
  \end{align*}
  Since $\mu$ is locally finite and $\phi$ is continuous with $\phi(0)=0$, we choose
  \begin{equation*}
    0<\delta^{\prime}< \min\Bigl\{1,\;\tilde\lambda\,\phi^{-1}\Bigl(\tfrac{1}{2\,\mu(J)}\Bigr)\Bigr\},
  \end{equation*}
  such that
  \begin{equation*}
    \mu(J)\phi\Bigl(\frac{\delta^{\prime}}{\tilde\lambda}\Bigr)<\frac 1 2.
  \end{equation*}
  Thus, on $J$ we have
  \begin{equation*}
    \int_J\phi\Bigl(\frac{\|h(x)-G(x)\|}{\tilde\lambda}\Bigr)\dd\mu(x)<\frac 1 2.
  \end{equation*}

  On $K\setminus J$, $G$ is bounded by $\max\{|c|,|C|\}$ by construction. Therefore, we obtain by \eqref{eq: 3} and the monotonicity of $\phi$
  \begin{align*}
    \int_{K\setminus J}\phi\Bigl(\frac{\|h(x)-G(x)\|}{\tilde\lambda}\Bigr)\dd\mu(x)&\le \mu(K\setminus J)\sup_{x\in K\setminus J}\phi\Bigl(\frac{\|h(x)-G(x)\|}{\tilde\lambda}\Bigr)\\
    &\le \mu(K\setminus J)\phi\Bigl(\frac{2\max\{|c|,|C|\}}{\tilde\lambda}\Bigr)\\
    &<\frac 1 2.
  \end{align*}
  Lastly, since $G$ and $h$ vanish outside of $K$, we have
  \begin{equation*}
    \int_{\R^{N_0}\setminus K}\phi\Bigl(\frac{\|h(x)-G(x)\|}{\tilde\lambda}\Bigr)\dd\mu(x)=0.
  \end{equation*}
  Putting everything together, yields
  \begin{align*}
    &\int_{\R^{N_0}}\phi\Bigl(\frac{\|h(x)-G(x)\|}{\tilde\lambda}\Bigr)\dd\mu(x)\\
    &=\int_{\R^{N_0}\setminus K}\phi\Bigl(\frac{\|h(x)-G(x)\|}{\tilde\lambda}\Bigr)\dd\mu(x)+\int_{J}\phi\Bigl(\frac{\|h(x)-G(x)\|}{\tilde\lambda}\Bigr)\dd\mu(x)\\
    &\qquad+\int_{K\setminus J}\phi\Bigl(\frac{\|h(x)-G(x)\|}{\tilde\lambda}\Bigr)\dd\mu(x)\\
    &<0+\frac 1 2+\frac 1 2\\
    &=1.
  \end{align*}
  Thus, we have by definition of $N_{\phi,\mu}$ and since $\tilde\lambda=\frac \epsilon 2$,
  \begin{equation}
    N_{\phi,\mu}(h-G)< \frac \epsilon 2.\label{eq: 5}
  \end{equation}
  Hence, by equations \eqref{eq: 1} and \eqref{eq: 5},
  \begin{equation*}
    N_{\phi,\mu}(f-G)\le N_{\phi,\mu}(f-h)+N_{\phi,\mu}(h-G)<\epsilon,
  \end{equation*}
  which yields the assertion.
\end{proof}

\begin{remark}
  In particular, Proposition~\ref{prop: Orlicz ReLU} extends the universal approximation theorem in $L^p$, presented in \cite[Theorem~4.16]{Kidger2020}, which treats only the Lebesgue measure on $\R^{N_0}$, to the setting of Orlicz spaces, allowing even for arbitrary locally finite Borel measures. Indeed, choosing $\phi(x):=|x|^p$, $x\in \R$ for $p\in [1,\infty)$, we immediately obtain, by Remark~\ref{rem: Lebesgue spaces}, the following $L^p$-universal approximation theorem:

  Let $\mu$ be a locally finite Borel measure on $\R^{N_0}$, let $\varrho\colon\R\to\R$ be the ReLU activation function, and $p\in[1,\infty)$. Then, $\cN_{N_0,N_L,\infty,N_0+N_{L}+1}^\varrho$ is dense in $L^p(\mu)$.
\end{remark}

\subsection{Neural Networks with Non-Polynomial Activation Function}

In this subsection we derive a universal approximation theorem in Orlicz spaces, allowing for feedforward neural networks with non-polynomial activation functions and allowing for compactly supported, finite Borel measures on the Euclidean space $\R^{N_0}$.

\begin{proposition}\label{prop: UAT Orlicz nonpolynomial}
 Let $L\ge 2$. Let $\phi$ be an $N$-function and let $\varrho\in C(\R;\R)$ be a non-polynomial activation function. Then, $\cN^\varrho_{N_0,N_L,L,\infty}$ is dense in $M^\phi(\mu)$ for every finite Borel measure on $(\R^{N_0},\mathcal B(\R^{N_0}))$ with compact support.
\end{proposition}

\begin{proof}
  Let $K$ denote the support of $\mu$. Since, by Proposition~\ref{prop: compact dense in M phi} $C_c(\R^{N_0};\R^{N_L})\subset M^\phi(\mu)$ is dense, it follows that $C(K;\R^{N_L})\subset M^\phi(\mu)$ is dense with respect to the gauge norm $N_{\phi,\mu}$, that is, for $f\in M^\phi(\mu)$ there exists a function $g\in C(K;\R^{N_L})$ such that
  \begin{equation*}
    N_{\phi,\mu}(f-g)\le \frac \epsilon 2.
  \end{equation*}
  Then, by classical universal approximation theorems see e.g., \cite{Pinkus1999, Cybenko1989}, there exists a neural network $\eta\in\cN^\varrho_{N_0,N_L,L,\infty}$ such that
  \begin{equation}\label{eq: approx on compacta non-polynomial}
    \sup_{x\in K}\|g(x)-\eta(x)\|<\delta,
  \end{equation}
  where $\delta>0$ will be chosen small in a moment. Note that even though the classical UAT is only stated for single-hidden layer networks, one can use that each additional hidden layer can be chosen to approximate the identity on compact sets. The reason is that the identity map is continuous on compact sets. Hence, since $\varrho$ is non-polynomial, the UAT applies to the coordinate functions of the identity map on any compact set, see \cite{Leshno1993}.

  Setting $\lambda:=\frac \epsilon 2$, by the monotonicity of $\phi$ and \eqref{eq: approx on compacta non-polynomial}, we have
  \begin{align*}
    \int_{\R^{N_0}}\phi\Bigl(\frac{\|g(x)-\eta(x)\|}{\lambda}\Bigr)\dd\mu(x)
    &=\int_K\phi\Bigl(\frac{\|g(x)-\eta(x)\|}{\lambda}\Bigr)\dd\mu(x)\\
    &\le \mu(K)\sup_{x\in K}\phi\Bigl(\frac{\|g(x)-\eta(x)\|}{\lambda}\Bigr)\\
    &\le \mu(K)\phi\Bigl(\frac{\delta}{\lambda}\Bigr).
  \end{align*}
  Since $\mu$ is finite and $\phi$ is continuous with $\phi(0)=0$, we choose
  \begin{equation*}
    0<\delta<\lambda \phi^{-1}\Bigl(\frac{1}{\mu(K)}\Bigr),
  \end{equation*}
  such that $\mu(K)\phi\Bigl(\frac \delta\lambda\Bigr)<1$. Hence, we have by the definition of the gauge norm  $N_{\phi,\mu}$
  \begin{equation*}
    N_{\phi,\mu}(g-\eta)<\frac\epsilon 2.
  \end{equation*}
  Therefore, we obtain
  \begin{equation*}
    N_{\phi,\mu}(f-\eta)\le N_{\phi,\mu}(f-g)+N_{\phi,\mu}(g-\eta)< \epsilon.
  \end{equation*}
\end{proof}

\begin{remark}
  Many classical universal approximation theorems, dealing with the locally uniform approximation of functions, allow naturally to deduce universal approximation theorems in $L^p$ with respect to compactly supported, finite reference measures $\mu$ on the Euclidean space~$\R^{N_0}$. For instance, a classical universal approximation theorem in $L^p$ can be found in Proposition~2 of \cite{Leshno1993}. Proposition~\ref{prop: UAT Orlicz nonpolynomial} generalizes the aforementioned observation to the setting of Orlicz spaces.
\end{remark}

\subsection{Functional Input Neural Networks}

In this subsection we derive a universal approximation theorem in Orlicz spaces for functional input neural networks, allowing for finite Borel measures on the Euclidean space $\R^{N_0}$. More precisely, to accommodate a broader class of activation functions, we consider functional input neural networks based on additive families $\mathcal H\subseteq\mathcal B_\psi(\R^{N_0})$ and assume that the activation function $\varrho$ is $\psi_1$-activating.

\begin{proposition}\label{prop: UAT functional}
  Let $\phi$ be an $N$-function. Let $\mathcal H\subseteq \mathcal B_\psi(\R^{N_0})$ be an additive family such that
  \begin{equation*}
    \sup\limits_{x\in\R^{N_0}}\frac{\psi_1(h(x))}{\psi(x)}<\infty,
    \quad \text{for all } h\in\mathcal H,
  \end{equation*}
  and $\varrho\in C(\R;\R)$ be $\psi_1$-activating. If $N_{\phi,\mu}(\psi)<\infty$, then $\cN_{\R^{N_0},\R^{N_2}}^{\mathcal H,\varrho}$ is dense in $M^\phi(\mu)$ with respect to the gauge norm $N_{\phi,\mu}$ for every finite Borel measure $\mu$ on $\R^{N_0}$.
\end{proposition}

\begin{proof}
  Let $f\in M^\phi( \mu)$. Fix $\epsilon>0$. For simplicity assume $\R^{N_2}$ is endowed with the $\|\cdot\|_\infty$ norm and for notational convenience we write $\|\cdot\|$.

  Step~1: For any $K>0$, we can define the function $f_K(x):=\mathbf{1}_{\{\|f(x)\|\le K\}}(x)f(x)$ for which we have $N_{\phi,  \mu}(f-f_K)\to 0$ as $K\to \infty$ by dominated convergence. Therefore, there is a $K^{\epsilon}>0$ such that $N_{\phi,  \mu}(f-f_{K^{\epsilon}})\le\frac{\epsilon}{3}$.

  Step~2: By Lusin's theorem \cite[Theorem~2.5.17]{Denkowski2003}, there is a closed set $C^\epsilon\subset \R^{N_0}$, such that $f_{K^\epsilon}$ restricted to $C^\epsilon$ is continuous and
  \begin{equation*}
    \mu(\R^{N_0}\setminus C^\epsilon)\le \phi\Bigl(\frac{2K^\epsilon}{\lambda}\Bigr)^{-1}.
  \end{equation*}
  By Tietze's extension theorem \cite[Theorem~3.6.3]{Friedman1982}, there is a continuous extension $f^\epsilon\in C_b(\R^{N_0};[-K^\epsilon,K^\epsilon])$ of $f_{K^\epsilon},$ such that
  \begin{equation*}
    \int_{\R^{N_0}\setminus C^\epsilon}\phi\Bigl(\frac{\|f_{K^\epsilon}-f^\epsilon\|}{\lambda}\Bigr)\dd \mu\le \mu(\R^{N_0}\setminus C^\epsilon)\phi\Bigl(\frac{2K^\epsilon}{\lambda}\Bigr)\le 1.
  \end{equation*}
  Setting $\lambda=\frac \epsilon 3$ yields
  \begin{equation*}
    N_{\phi,  \mu}(f_{K^\epsilon}-f^\epsilon)\le\frac \epsilon 3.
  \end{equation*}

  Step~3: Moreover, since $C_b\subseteq \mathcal B_\psi$, by Theorem~4.13 of \cite{Cuchiero2024} we can approximate $f^\epsilon$ by some $\eta\in\cN_{\R^{N_0},\R^{N_2}}^{\mathcal H,\varrho}$. More precisely, let $M:=N_{\phi,\mu}(\psi)<\infty$, then we have
  \begin{equation*}
    \|{f}^\epsilon-\eta\|_{\mathcal B_\psi(\R^{N_0};\R^{N_2})}=\sup_{x\in \R^{N_0}}\frac{\|{f}^\epsilon(x)-\eta(x)\|}{\psi(x)}
    < \frac{\epsilon}{3M}.
  \end{equation*}
  Hence, we get
  \begin{align*}
    N_{\phi,  \mu}(f^\epsilon-\eta)&=N_{\phi,  \mu}\Bigl(\frac{f^\epsilon-\eta}{\psi}\psi\Bigr)\le N_{\phi,  \mu}(\|f^\epsilon-\eta\|_{\mathcal B_\psi(\R^{N_0};\R^{N_2})}\psi)\\
    &=\|f^\epsilon-\eta\|_{\mathcal B_\psi(\R^{N_0};\R^{N_2})}N_{\phi,  \mu}(\psi)<\frac \epsilon 3.
  \end{align*}

  Combining Step~1 to~3, we obtain
  \begin{equation*}
    N_{\phi, \mu}(f-\eta)\le N_{\phi, \mu}(f-f_{K^\epsilon})+N_{\phi, \mu}(f_{K^\epsilon}-f^\epsilon)+
    N_{\phi, \mu}({f}^\epsilon-\eta)< \epsilon,
  \end{equation*}
  which concludes the proof.
\end{proof}

\begin{remark}
  \begin{itemize}
    \item Observe that this approximation result captures in particular the feedforward neural networks associated with the additive family $\widetilde{\mathcal H}$ as given in \eqref{eq: additive family}.
    \item Choosing the $N$-function $\phi(x)=|x|^p$ for $x\in\R$ recovers the classical $L^p$-space setting, under the assumption that the weight function $\psi$ satisfies $\int_X \psi^p \dd\mu < \infty$.
    \item The universal approximation property for functional input neural networks is of special interest, as it allows for a broader class of activation functions, including sigmoidal or $\psi_1$-discriminatory activations. For a detailed discussion of admissible activation functions, we refer to  \cite[Proposition~4.4 and Remark~4.5]{Cuchiero2024}.
  \end{itemize}
\end{remark}

\section{Distributionally Robust Approximation Property}\label{sec: Robust UAT}

In this section, we deduce the approximation property of feedforward and functional input neural networks with respect to a weakly compact family of finite Borel measures. To that end, let us start by introducing the topology on the space of measures which is used throughout this section.

Let $X:=\R^{N_0}$ and $\mathcal A:=\mathcal B(\R^{N_0})$. Following \cite[Section~4.7~(v)]{Bogachev2007}, we denote by $\mathfrak M(X,\mathcal A)$ the space of all bounded countably additive signed measures on $(X,\mathcal A)$, endowed with the total variation norm, Here, bounded means that the total variation is finite. Let $\mathcal F$ be the linear space of all simple $\mathcal A$-measurable functions. We consider on $\mathfrak M(X,\mathcal A)$ the topology $\sigma(\mathfrak M,\mathcal F)$, that is, the topology of setwise convergence of measures. By \cite[Section~4.7~(v)]{Bogachev2007}, the compact sets for this topology coincide with the compact sets for the usual weak topology $\sigma(\mathfrak M,\mathfrak M^\ast)$ of the Banach space $\mathfrak M$. Therefore, throughout this section, a family $\mathcal M$ of finite Borel measures on $X$ is called weakly compact if it is compact in the topology $\sigma(\mathfrak M,\mathcal F)$.

Let now $\mathcal{M}$ be a weakly compact set of finite Borel measures on $(\R^{N_0},\mathcal B(\R^{N_0}))$ in the above sense. Then, all measures $\nu\in\mathcal{M}$ are absolutely continuous with respect to some finite Borel measure $\mu_{\mathcal{M}}$ on $\R^{N_0}$, i.e., $\nu\ll\mu_{\mathcal{M}}$ and $\mu_{\mathcal{M}}(\R^{N_0})<\infty$, see e.g., \cite[4.7.25.~Theorem]{Bogachev2007}. Moreover, the set
\begin{equation*}
  \bigg\{\frac{\d\nu}{\d\mu_{\mathcal{M}}}:\, \nu\in\mathcal M\bigg\}
\end{equation*}
is uniformly integrable in $L^1(\mu_{\mathcal{M}})$ by Theorem~4.7.25 of \cite{Bogachev2007}, where $\frac{\d\nu}{\d\mu_{\mathcal{M}}}$ denotes the Radon--Nikodym density of $\nu$ with respect to $\mu_{\mathcal{M}}$. Hence, by the De la Vallée Poussin Theorem (Theorem~\ref{thm: vallee}) there exists an $N$-function~$\psi_{\mathcal{M}}$ such that
\begin{equation*}
  \sup_{\nu\in\mathcal M}N_{\psi_{\mathcal{M}},\mu_{\mathcal{M}}}\Bigl(\frac{\d\nu}{\d\mu_{\mathcal{M}}}\Bigr)<\infty.
\end{equation*}
We denote by $\phi_{\mathcal{M}}$ the complementary function to $\psi_{\mathcal{M}}$, see Definition~\ref{def: Young function}, and in this case we say that $(\phi_{\mathcal{M}},\psi_{\mathcal{M}})$ is the \emph{associated Young pair} to $\mathcal{M}$.

 Before presenting a universal approximation theorem, allowing to approximate functions uniformly on weakly compact sets of measures, let us briefly discuss the used notation of weakly compactness and provide some examples.

\begin{remark}
  The compactness assumption used in the presented work should not be confused with compactness in the frequently used weak (or narrow) topology on measures, namely the topology induced by convergence against bounded continuous test functions.

  On Polish spaces, Prokhorov's theorem characterizes relative compactness in the narrow topology by tightness. This notion is generally weaker than compactness in the setwise topology $\sigma(\mathfrak M,\mathcal F)$ used here, and it does not imply the existence of a common finite dominating measure. Similarly, compactness in the $C_0$-weak-$\ast$-topology on Radon measures, that is, the weak-$\ast$-topology induced by testing against continuous functions vanishing at infinity, is not sufficient for the present result.

  For example, let $K=[0,1]$ and consider $\mathcal M:=\{\delta_x: x\in K\}$, where $\delta_x$ denotes the Dirac measure at $x$. Then, $\mathcal M$ is compact in the usual narrow topology of finite Borel measures on $K$, since the map $x\mapsto\delta_x$ is continuous and $K$ is compact. However, $\mathcal M$ does not admit a common finite dominating measure. Indeed, if a finite measure $\mu$ satisfied $\delta_x\ll\mu$ for every $x\in K$, then necessarily $\mu(\{x\})>0$ for every $x\in K$, which is impossible for a finite measure on an uncountable set. Moreover, $\mathcal M$ is not compact in the setwise topology. In fact, the induced setwise topology on $\mathcal M$ is discrete, since the singleton $\{\delta_x\}$ is separated by the measurable set $\{x\}$. Thus an uncountable family of Dirac measures is not compact in the setwise topology.

  Consequently, robust $L^1$-approximation over $\mathcal M$ reduces to uniform approximation on $K$. Taking, for instance, $f=\mathbf 1_{[1/2,1]}$, no continuous function can approximate $f$ uniformly with error strictly smaller than $1/2$. Hence, neural networks with continuous activation functions cannot approximate all bounded measurable functions uniformly over such a non-dominated family. In particular, Prokhorov compactness, or compactness in the $C_0$-weak-$\ast$-topology alone, would be not sufficient for Theorem~\ref{thm: Robust UAT} holds true.
\end{remark}

\begin{example}
  Let $X:=\R^{N_0}$ and $\mathcal A:=\mathcal B(\R^{N_0})$. Let $\mu_{\mathcal M}$ be a finite Borel measure on $(X,\mathcal A)$.

  \begin{itemize}
    \item [(i)] Let $\psi$ be an $N$-function and $C\ge 1.$ Consider
      \begin{equation*}
        \mathcal M:=\Bigl\{\nu\in\mathfrak M(X,\mathcal A):\nu\ge 0,\nu\ll \mu_{\mathcal M},\int_X\psi\Bigl(\frac{\d\nu}{\d\mu_{\mathcal M}}\Bigr)\dd\mu_{\mathcal M} \le C\Bigr\}.
      \end{equation*}
      Then $\mathcal M$ is weakly compact in the sense of $\sigma(\mathfrak M,\mathcal F)$ and the Orlicz bound implies, by the de la Vallée Poussin criterion, that the family of Radon--Nikodym densities
      \begin{equation*}
        \Bigl\{\frac{\d\nu}{\d\mu_{\mathcal M}}:\nu\in\mathcal M\Bigr\}
      \end{equation*}
      is uniformly integrable in $L^1(\mu_{\mathcal M})$. Hence, by the Dunford--Pettis theorem, this family is relatively weakly compact in $L^1(\mu_{\mathcal M})$. By Bogachev's characterization of weak compactness, this implies that $\mathcal M$ is relatively compact in $\sigma(\mathfrak M,\mathcal F)$. Moreover, $\mathcal M$ is closed in this topology, since positivity, absolute continuity with respect to $\mu_{\mathcal M}$, and the above Orlicz bound are preserved under setwise limits. Therefore, $\mathcal M$ is weakly compact.

     \medskip
     The preceding class of weakly compact measure families is well suited to model drift uncertainty (see for instance \cite{Sass2021}). Indeed, let $\mu_{\mathcal M}$ be interpreted as a nominal reference law, for instance the law of a Brownian path under a reference probability measure. Possible drift perturbations of the underlying model lead, by Girsanov's theorem, to alternative measures $\nu$ which are absolutely continuous with respect to $\mu_{\mathcal M}$. The Orlicz bound on the likelihood ratios $\frac{\d\nu}{\d\mu_{\mathcal M}}$ then quantifies the admissible size of the distributional drift.

     \item[(ii)] Assume in addition that $\mu_{\mathcal M}$ is a probability measure. For $\nu\ll\mu_{\mathcal M}$, define the relative entropy of $\nu$ with respect to $\mu_{\mathcal M}$ by
     \begin{equation*}
       H(\nu\mid\mu_{\mathcal M}):=\int_X\frac{\d\nu}{\d\mu_{\mathcal M}}\log\Bigl(\frac{\d\nu}{\d\mu_{\mathcal M}}\Bigr)\dd\mu_{\mathcal M},
     \end{equation*}
     with the convention $0\log 0:=0$. For $C\ge0$, consider
     \begin{equation*}
       \mathcal M:=\Bigl\{\nu\in\mathfrak M(X,\mathcal A):\nu\ge0,\nu\ll\mu_{\mathcal M},\nu(X)=1,H(\nu\mid\mu_{\mathcal M})\le C\Bigr\}.
     \end{equation*}
     Then, $\mathcal M$ is weakly compact in the topology $\sigma(\mathfrak M,\mathcal F)$.

     Although the entropy integrand $x\mapsto x\log x$ is not an $N$-function, the entropy bound implies a uniform Orlicz bound for an entropy-growth $N$-function. For instance, define $\psi(x):=(1+|x|)\log(1+|x|)-|x|$, $x\in\R.$ Then $\psi$ is an $N$-function and the bound $H(\nu\mid\mu_{\mathcal M})\le C$, together with $\nu(X)=1$, yields
     \begin{equation*}
       \sup_{\nu\in\mathcal M}\int_X\psi\Bigl(\frac{\d\nu}{\d\mu_{\mathcal M}}\Bigr)\dd\mu_{\mathcal M}<\infty.
     \end{equation*}
     Hence, by the de la Vallée Poussin theorem, the family of densities
     \begin{equation*}
       \Bigl\{\frac{\d\nu}{\d\mu_{\mathcal M}}:\nu\in\mathcal M\Bigr\}
     \end{equation*}
     is uniformly integrable in $L^1(\mu_{\mathcal M})$. Since all measures in $\mathcal M$ are dominated by the finite measure $\mu_{\mathcal M}$, Bogachev's characterization of weak compactness for dominated families of measures implies that $\mathcal M$ has compact closure in $\sigma(\mathfrak M,\mathcal F)$. Moreover, $\mathcal M$ is closed in this topology by lower semicontinuity of the relative entropy under setwise convergence. Therefore, $\mathcal M$ is weakly compact. 
     
   This class is motivated by relative-entropy ambiguity sets in distributionally robust optimization. In such settings, $\mu_{\mathcal M}$ represents a nominal model, while the measures $\nu\in\mathcal M$ represent alternative models whose relative entropy with respect to the nominal law is bounded. Equivalently, writing $Z=d\nu/d\mu_{\mathcal M}$, the constraint is of the form $\int_X Z\log Z\,d\mu_{\mathcal M}\le C$. Such relative-entropy, or Kullback--Leibler, ambiguity sets are standard in distributionally robust optimization, and closely related entropy constraints and penalties are widely used in robust control under model misspecification, see, e.g., \cite{Hansen2001,Hu2013}.
  \end{itemize}
\end{example}

The following theorem presents the universal approximation property uniformly with respect to weakly compact families of measures for several classes of neural networks.

\begin{theorem}[Robust UAT for Neural Networks]\label{thm: Robust UAT}
  Let $\mathcal M$ be a weakly compact set of finite Borel measures on $\R^{N_0}$ with associated Young pair $(\phi_{\mathcal M},\psi_{\mathcal M})$ and\\
  let $f\in M^{\phi_{\mathcal M}}(\mu_{\mathcal M};\R^{N_L})$. Then, the following statements hold:
  \begin{enumerate}
    \item[(i)] If $\varrho\in C(\R;\R)$ is bounded and non-constant, then for every $L\ge 2$ and every $\epsilon>0$ there exists $\eta\in\cN_{N_0,N_L,L,\infty}^\varrho$ such that
    \begin{equation*}
      \sup_{\nu\in\mathcal M}\|f-\eta\|_{L^1(\nu;\R^{N_L})}< \epsilon.
    \end{equation*}
    \item[(ii)] If $\varrho$ is the ReLU activation, then for every $\epsilon>0$ there exists $\eta\in\cN_{N_0,N_L,\infty,N_0+N_L+1}^\varrho$ such that
    \begin{equation*}
      \sup_{\nu\in\mathcal M}\|f-\eta\|_{L^1(\nu;\R^{N_L})}< \epsilon.
    \end{equation*}
    \item[(iii)] If, in addition, $\operatorname{supp}\nu\subset K$ for all $\nu\in\mathcal M$ with some compact $K\subset \R^{N_0}$, then for every non-polynomial $\varrho\in C(\R;\R)$, every $L\ge 2$ and every $\epsilon>0$ there exists $\eta\in\cN^\varrho_{N_0,N_L,L,\infty}$ such that
    \begin{equation*}
      \sup_{\nu\in\mathcal M}\|f-\eta\|_{L^1(\nu;\R^{N_L})}< \epsilon.
    \end{equation*}
    \item[(iv)] If $L=2$, $\mathcal H\subseteq\mathcal B_{\psi_{\mathcal M}}(\R^{N_0})$ is an additive family and $\varrho\in C(\R;\R)$ is $\psi_1$-activating with $\sup_{x\in\R^{N_0}}\frac{\psi_1(h(x))}{\psi_{\mathcal M}(x)}<\infty$, for all $h\in\mathcal H$ and $N_{\phi_{\mathcal M},\mu_{\mathcal M}}(\psi_{\mathcal M})<\infty$, then for every $\epsilon>0$ there exists $\eta\in\cN_{\R^{N_0},\R^{N_2}}^{\mathcal H,\varrho}$ such that
    \begin{equation*}
      \sup_{\nu\in\mathcal M}\|f-\eta\|_{L^1(\nu;\R^{N_2})}< \epsilon.
    \end{equation*}
  \end{enumerate}
\end{theorem}

\begin{proof}
  (i) Let $K:=\sup_{\nu\in\mathcal{M}}N_{\psi_{\mathcal{M}},\mu_{\mathcal{M}}}\Bigl(\frac{\dd\nu}{\dd\mu_{\mathcal{M}}}\Bigr)<\infty$. By Proposition~\ref{prop: UAT bounded} we have that for $\epsilon>0$ there exists an $\eta\in\cN^\varrho_{N_0,N_L,L,\infty}$ such that $N_{\phi_{\mathcal{M}},\mu_{\mathcal{M}}}(f-\eta)<\frac{\epsilon}{2K}$. Therefore, using the generalized H{\"o}lder's inequality (Proposition~\ref{prop: genHölder}), we obtain
  \begin{align*}
    \sup_{\nu\in \mathcal M}\|f-\eta\|_{L^1(\nu;\R^{N_L})}&=\sup_{\nu\in \mathcal M}\int_{\R^{N_0}}\|f-\eta\|\dd\nu\\
    &=\sup_{\nu\in \mathcal M}\int_{\R^{N_0}}\|f-\eta\|\frac{\d\nu}{\d\mu_{\mathcal{M}}}\dd\mu_{\mathcal M}\\
    &\le 2 N_{\phi_{\mathcal{M}},\mu_{\mathcal{M}}}(f-\eta)\sup_{\nu\in\mathcal M}N_{\psi_{\mathcal{M}},\mu_{\mathcal{M}}}\Bigl(\frac{\d\nu}{\d\mu_{\mathcal{M}}}\Bigr)\\
    &< \epsilon,
  \end{align*}
  which concludes the proof.

  (ii) and (iv) follow by the same argument as in the proof of (i), relying instead on Proposition~\ref{prop: Orlicz ReLU} and Proposition~\ref{prop: UAT functional}, respectively.

  The proof of (iii) follows also analogously to (i), invoking Proposition~\ref{prop: UAT Orlicz nonpolynomial} and using that $\mu_{\mathcal M}$ is a finite regular Borel measure with compact support.
\end{proof}

We immediately obtain the following corollary.

\begin{corollary}\label{cor:robust UAT}
  Let $\mathcal M$ be a weakly compact set of finite Borel measures on $\R^{N_0}$ with associated Young pair $(\phi_{\mathcal M},\psi_{\mathcal M})$, and let $f\colon \R^{N_0}\to\R^{N_L}$ be bounded and measurable. Then, for every $\epsilon>0$, there exists a neural network $\eta$ such that
  \begin{equation*}
    \sup_{\nu\in\mathcal M} \|f-\eta\|_{L^1(\nu;\R^{N_L})} < \epsilon,
  \end{equation*}
  where the class of neural networks is specified according to the following cases:
  \begin{enumerate}
    \item[(i)] if $\varrho$ is bounded and non-constant, then, for every $L\ge 2$, $\eta\in\cN_{N_0,N_L,L,\infty}^\varrho$;
    \item[(ii)] if $\varrho$ is the ReLU activation, then $\eta\in\cN_{N_0,N_L,\infty,N_0+N_L+1}^\varrho$;
    \item[(iii)] if $\varrho\in C(\R;\R)$ is non-polynomial and $\operatorname{supp}\nu \subset K$ for all $\nu\in\mathcal M$ for a compact $K\subset \R^{N_0}$, then, for every $L\ge 2$, $\eta\in\cN_{N_0,N_L,L,\infty}^\varrho$;
    \item[(iv)] if $L=2$, $\mathcal H\subseteq\mathcal B_{\psi_{\mathcal M}}(\R^{N_0})$ is an additive family and $\varrho\in C(\R;\R)$ is $\psi_1$-activating with $\sup_{x\in \R^{N_0}} \frac{\psi_1(h(x))}{\psi_{\mathcal M}(x)}<\infty$ for all $h\in\mathcal H$ and $N_{\phi_{\mathcal M},\mu_{\mathcal M}}(\psi_{\mathcal M})<\infty$, then $\eta\in\cN_{\R^{N_0},\R^{N_2}}^{\mathcal H,\varrho}$.
  \end{enumerate}
\end{corollary}

\begin{remark}
  Corollary~\ref{cor:robust UAT} applies, in particular, to $f\in C_b(\R^{N_0};\R^{N_L})$, and for suitable Young pairs $(\phi_{\mathcal M},\psi_{\mathcal M})$, to $f\in L^p(\mu;\R^{N_L})$ with $1 = \frac{1}{p} + \frac{1}{q}$ for $p,q \in (1,\infty)$, for instance, if $\psi_{\mathcal M}(x) = |x|^q$, $x\in \R$.
\end{remark}

\section{UATs for Functional Input Neural Networks in Infinite Dimension}\label{sec: infinite}

In this section, we briefly discuss the extension of the universal approximation theorems presented in Sections~\ref{sec: UAT in Orlicz spaces} and~\ref{sec: Robust UAT} to infinite-dimensional input and output spaces in the context of functional input neural networks. Note that the underlying theory of Orlicz spaces and functional input neural networks naturally extends to the setting of infinite-dimensional Banach spaces. By contrast, the other classes of neural networks considered in this work could, in principle, also be extended to the infinite-dimensional setting, but doing so would require substantially more extensive modifications of both the framework and the underlying theory.

To that end, let $X$ be a Polish space and $(Y,\|\cdot\|_Y)$ a separable Banach space. Whenever $\mu$ is a finite Borel measure on $X$ and $\phi$ is an $N$-function, we consider the corresponding $Y$-valued Orlicz space $\mathcal L^\phi(\mu;Y)$ and Orlicz heart $M^\phi(\mu;Y)$, defined as before with $\R^{N_0}$ replaced by $X$, $\R^{N_L}$ replaced by $Y$, and the Euclidean norm replaced by $\|\cdot\|_Y$, see e.g., \cite{Schappacher2005} for a general treatment of Banach-valued Orlicz spaces.

Then, similar as in Subsection~\ref{subsec: essentials on NN}, a function $\psi\colon X\to (0,\infty)$ is called an admissible weight function if every pre-image $K_R:=\psi^{-1}((0,R])$ is compact for all $R>0$ and we call the pair $(X,\psi)$ a weighted space.
Moreover, we consider the function space
\begin{equation*}
B_\psi(X;Y):=\Bigl\{f\colon X\to Y: \sup_{x\in X}\frac{\|f(x)\|_Y}{\psi(x)}<\infty\Bigr\}.
\end{equation*}
We equip this vector space with the weighted norm $\|\cdot\|_{\mathcal B_\psi(X;Y)}$ given by
\begin{equation*}
\|f\|_{\mathcal B_\psi(X;Y)}=\sup_{x\in X}\frac{\|f(x)\|_Y}{\psi(x)},\quad \text{ for } f\in B_\psi(X;Y).
\end{equation*}
We further set $\mathcal B_\psi(X;Y):=\overline{C_b(X;Y)}^{\|\cdot\|_{\mathcal B_\psi(X;Y)}}$, which is a Banach space.
Then, an additive family is a subset $\mathcal H\subseteq \mathcal B_\psi(X)$, satisfying the same conditions as given above in Subsection~\ref{subsec: essentials on NN}.

Now the functional input neural networks based on the weighted input space $(X,\psi)$, the weighted hidden space $(\R,\psi_1)$ and the output space $(Y,\|\cdot\|_Y)$ are defined as follows:
For a given additive family $\mathcal H\subseteq \mathcal B_\psi(X)$, a function $\varrho\in C(\R;\R)$ and a vector subspace $\mathcal L\subseteq Y$, a functional input neural network $\eta\colon X\to Y$ is defined as functions of the form
\begin{equation*}
\eta(x)=\sum_{n=1}^Ny_n\varrho(h_n(x)),\quad \text{ for } x\in X,
\end{equation*}
where $N\in\N$ denotes the number of neurons, $h_1,\ldots,h_N\in\mathcal H$ are the hidden layer maps, and $y_1,\ldots,y_N\in\mathcal L$ represent the linear readouts. We then denote by $\cN^{\mathcal H,\varrho,\mathcal L}_{X,Y}$ the set of FNNs of the above form.

\begin{proposition}\label{prop: infinite}
  Let $\phi$ be an $N$-function. Let $\mathcal H\subseteq \mathcal B_\psi(X)$ be an additive family such that
  \begin{equation*}
    \sup\limits_{x\in X}\frac{\psi_1(h(x))}{\psi(x)}<\infty,
    \quad \text{for all } h\in\mathcal H,
  \end{equation*}
  and $\varrho\in C(\R;\R)$ be $\psi_1$-activating and let $\mathcal L\subseteq Y$ be a vector subspace that is dense in $Y$. If $N_{\phi,\mu}(\psi)<\infty$, then $\cN_{X,Y}^{\mathcal H,\varrho,\mathcal L}$ is dense in $M^\phi(\mu;Y)$ with respect to the gauge norm $N_{\phi,\mu}$ for every finite Borel measure $\mu$ on $X$.
\end{proposition}

\begin{proof}
The proof proceeds along the same lines as the proof of Proposition~\ref{prop: UAT functional}. One first truncates the given function and then applies Lusin's theorem \cite[Theorem~2.5.15]{Denkowski2003} to obtain a compact set on which the truncated function is continuous. Since compact subsets of Polish spaces are closed, the Dugundji extension theorem (see \cite{Dugundji1951}) applies and provides a bounded continuous extension of this restriction to the whole space, with values in the Banach space under consideration. The remaining step is the neural network approximation of this continuous extension, carried out with respect to the weighted norm $\|\cdot\|_{\mathcal B_\psi(X;Y)}$.
\end{proof}

For the robust UAT in the present infinite-dimensional setting, the definition of weak compactness for families of finite Borel measures remains the same, namely compactness in the setwise topology $\sigma(\mathfrak M(X,\mathcal B(X)),\mathcal F)$. Hence, we obtain an analogous robust universal approximation in this infinite dimensional setting.

\begin{corollary}
  Let $\mathcal M$ be a weakly compact set of finite Borel measures on $X$ with associated Young pair $(\phi_{\mathcal M},\psi_{\mathcal M})$ and let $f\in M^{\phi_{\mathcal M}}(\mu_{\mathcal M};Y)$ and let $\mathcal L\subseteq Y$ be a vector subspace that is dense in $Y$. If $\mathcal H\subseteq \mathcal B_{\psi_{\mathcal M}}(X)$ is an additive family and $\varrho\in C(\R;\R)$ is $\psi_1$-activating with $\sup_{x\in X}\frac{\psi_1(h(x))}{\psi_{\mathcal M}(x)}<\infty$, for all $h\in\mathcal H$ and $N_{\phi_{\mathcal M},\mu_{\mathcal M}}(\psi_{\mathcal M})<\infty$, then for every $\epsilon>0$ there exists $\eta\in\cN_{X,Y}^{\mathcal H,\varrho,\mathcal L}$ such that
  \begin{equation*}
  \sup_{\nu\in\mathcal M}\|f-\eta\|_{L^1(\nu;Y)}<\epsilon.
  \end{equation*}
  \end{corollary}
  \begin{proof}
  The proof follows analogously to the proof of Theorem~\ref{thm: Robust UAT}, using Proposition~\ref{prop: infinite}. The remaining argument is unchanged: the setwise compactness of $\mathcal M$ yields a common dominating measure $\mu_{\mathcal M}$ with uniformly integrable Radon--Nikodym densities, and the conclusion follows from the same generalized Hölder estimate.
  \end{proof}

\appendix
\section{Essentials on Orlicz Spaces}\label{sec:appendix}

In this appendix, we recall some basic concepts related to Orlicz spaces and their fundamental properties. To that end, we first need to discuss a few observations on certain convex functions that play a key role in the theory of Orlicz spaces. For a more comprehensive introduction to Orlicz spaces we refer, e.g., to \cite{Rao1991}.

\begin{definition}\label{def: Young function}
  A convex function $\phi\colon\R\to\R^+\cup\{\infty\}$ is called a \emph{Young function} if $\phi(0)=0$, $\phi(-x)=\phi(x)$ for $x\in \R$, and $\lim_{x\to\infty}\phi(x)=+\infty$. The convex function $\psi\colon\R\to\R^+\cup\{\infty\}$, defined by
  \begin{equation*}
    \psi(y):=\sup\{x|y|-\phi(x):\,x\ge 0\},\quad y\in\R,
  \end{equation*}
  is called the \emph{complementary function} to $\phi$. In this case $(\phi,\psi)$ is called a \emph{complementary Young pair}.
\end{definition}

From the definition of a Young function~$\phi$, it follows that its complementary function~$\psi$ satisfies $\psi(0)=0$ and $\psi(-y)=\psi(y)$ for $y \in \R$, as well as, that $\psi$ is a convex increasing function such that $\lim_{y\to +\infty}\psi(y)=+\infty$. Moreover, the pair $(\phi,\psi)$ fulfills the so called \emph{Young's inequality}:
\begin{equation*}
  xy\le \phi(x)+\psi(y),\quad x,y \in \R.
\end{equation*}

\begin{definition}\label{def: N-function}
  A continuous Young function $\phi\colon \R\to\R^+$ is called an \emph{$N$-function} if $\phi(x)=0$ if and only if $x=0$, $\lim_{x\to 0}\phi(x)/x=0$,  and $\lim_{x\to \infty}\phi(x)/x=+\infty$. A complementary Young pair $(\phi,\psi)$ is called a \emph{pair of complementary $N$-functions} if $\phi$ is an $N$-function.
\end{definition}

Let $(\Omega,\Sigma,\mu)$ be a finite measure space. We denote by $\tilde{\mathcal L}^\phi(\mu;\R^m)$ the set of all measurable functions $f\colon\Omega\to\R^m$ such that
\begin{equation*}
  \int_\Omega\phi(\|f\|)\dd\mu<\infty.
\end{equation*}
For a Young function~$\phi$, the \emph{Orlicz space} $\mathcal L^\phi(\mu; \R^m)$ is defined as the space of all measurable functions $f\colon \Omega\to\R^m$ such that $\alpha f\in\tilde{\mathcal L}^\phi(\mu;\R^m)$ for some $\alpha>0$, that is
\begin{equation*}
  \mathcal L^\phi(\mu;\R^m):=\bigg\{f\colon\Omega\to\R^m\,:\, \int_\Omega \phi(\alpha \|f\|)\dd\mu<\infty\,\text{for some}\,  \alpha>0\bigg\}.
\end{equation*}
Note that $\mathcal L^\phi(\mu;\R^m)$  is a vector space which we equip with the gauge norm
\begin{equation*}
  N_{\phi,\mu}(f):=\inf\Bigl\{k>0:\,\int_\Omega\phi\Bigl(\frac{\|f\|}{k}\Bigr)\dd\mu\le 1\Bigr\}.
\end{equation*}

For Orlicz spaces the following \emph{generalized H{\"o}lder's inequality} holds, see \cite[Chapter~3.3, Proposition~1 and Remark]{Rao1991}.

\begin{proposition}[Generalized H{\"o}lder's inequality]\label{prop: genHölder}
  Let $(\phi,\psi)$ be a complementary\\ Young pair and $k,m\in\N$. If $f\in\mathcal L^\phi(\mu;\R^m)$ and $g\in\mathcal L^\psi(\mu;\R^k)$, then one has
  \begin{equation*}
    \int_\Omega \|f\|\|g\|\dd\mu\le 2 N_{\phi,\mu}(f)N_{\psi,\mu}(g).
  \end{equation*}
\end{proposition}

Moreover, we set $\mathcal N_{\phi,\mu}:=\{f\in\mathcal L^\phi(\mu;\R^m):\,N_{\phi,\mu}(f)=0\}$. Then, $f\in\mathcal N_{\phi,\mu}$ if and only if $f=0$ a.e. Hence, the factor space $L^\phi(\mu;\R^m):=\mathcal L^\phi(\mu;\R^m)/\mathcal N_{\phi,\mu}$ with $\bar{N}_{\phi,\mu}([f]):=N_{\phi,\mu}(f)$ is well-defined, with $[f]:=\{g\in\mathcal L^\phi(\mu;\R^m):\, g-f\in\mathcal N_{\phi,\mu}\}$. This implies that $(\mathcal L^\phi(\mu;\R^m),\bar{N}_{\phi,\mu})$ is a Banach space,see \cite[Corollary~12, p. 69]{Rao1991}. With a slight abuse of notation, we use the symbol $N_{\phi,\mu}(\cdot)$ for $\bar{N}_{\phi,\mu}(\cdot)$ and note that $L^\phi(\mu;\R^m)$ is the Orlicz space $(\mathcal L^\phi(\mu;\R^m),N_{\phi,\mu}(\cdot))$ in which a.e.~equivalent functions are identified. In particular, $L^\phi(\mu;\R^m)$ is a Banach space. Similarly, we say that $\tilde{L}^\phi(\mu;\R^m)$ is the space $(\tilde{\mathcal L}^\phi(\mu;\R^m),N_{\phi,\mu}(\cdot))$ where a.e.~equal elements are identified.

\medskip

To formulate an analogue of the classical Riesz representation theorem in the setting of Orlicz spaces, we introduce
\begin{equation*}
  M^\phi(\mu;\R^m):=\bigg\{f\in L^\phi(\mu;\R^m)\,:\,kf\in\tilde{L}^\phi(\mu;\R^m)\,\text{for all}\,k>0\bigg\},
\end{equation*}
which is a linear subspace of $L^\phi(\mu;\R^m)$, known as the \emph{Orlicz heart}, see \cite[Definition~2.1.9]{Edgar1992}. Note that, taking $\phi(x):=|x|^p$ for $x\in\R$, we have \begin{equation*}
L^\phi(\mu;\R^m) =  M^\phi(\mu;\R^m) =  L^p(\mu;\R^m).
\end{equation*}

The next theorem characterizes the adjoint (dual) space of the Orlicz heart $M^\phi(\mu;\R^m)$, which we denote by $(M^\phi(\mu;\R^m))^\ast$, see \cite[Chapter~4.1, Theorem~7]{Rao1991}.

\begin{theorem}\label{thm: Riesz for Orlicz}
  Let $(\phi,\psi)$ be a pair of complementary $N$-functions. Then, one has
  \begin{equation*}
    (M^\phi(\mu;\R^m))^\ast=L^\psi(\mu;\R^m)
  \end{equation*}
  and, for each $x^\ast\in (M^\phi(\mu;\R^m))^\ast$, there is a unique $g_{x^\ast}\in L^\psi(\mu;\R^m)$ such that
  \begin{equation*}
    x^\ast(f)=\int_\Omega \langle f,g_{x^\ast}\rangle_{\R^m}\dd\mu,\quad f\in M^\phi(\mu;\R^m).
  \end{equation*}
\end{theorem}

\begin{remark}\label{rem: M phi}
  \begin{enumerate}
    \item[(i)] If the Young function $\phi\colon\R\to\R^+$ satisfies the $\Delta_2$-condition, i.e.,
    \begin{equation*}
      \phi(2x)\le K \phi(x),\quad x\ge x_0\ge 0,
    \end{equation*}
    for some absolute constant $K>0$, then, by \cite[Corollary~4.1.9]{Rao1991}, we have
    \begin{equation*}
      M^\phi(\mu;\R^m)=\tilde{L}^\phi(\mu;\R^m)=L^\phi(\mu;\R^m).
    \end{equation*}
    \item[(ii)] If $\phi$ is an $N$-function, then $M^\phi(\mu;\R^m)$ is the closure of simple step functions in $L^\phi(\mu;\R^m)$ under the gauge norm $N_{\phi,\mu}$, see \cite[p. 75]{Rao1991}. Moreover, in that case $(M^\phi(\mu;\R^m),N_{\phi,\mu})$ is a Banach space \cite[p. 76]{Rao1991}.
    \item[(iii)] Indeed, every bounded function is in $M^\phi(\mu;\R^m)$ as $\mu$ is a finite measure.
  \end{enumerate}
\end{remark}

Whenever the dimension of the image space is clear from context, we simply write $M^\phi(\mu)$ and $L^\phi(\mu)$. Moreover, we recall that $C_c(\R^n;\R^m)$ denotes the space of all continuous functions $f\colon \R^n\to \R^m$ with compact support.

\begin{proposition}\label{prop: compact dense in M phi}
  Let $\phi$ be an $N$-function. Then, $C_c(\R^n;\R^m)\subset M^\phi(\mu)$ is dense for every locally finite Borel measure $\mu$ on $(\R^n,\mathcal B(\R^n))$.
\end{proposition}

\begin{proof}
  For simplicity assume that $\R^{n}$ is endowed with the $\|\cdot\|_\infty$ norm and we write for notational convenience $\|\cdot\|$.

  Let $f\in M^\phi(\mu)$. As $\phi$ is an $N$-function, by Remark~\ref{rem: M phi}~(ii), $M^\phi(\mu)$ is the closure of simple step functions in $L^\phi(\mu)$ under the gauge norm $N_{\phi,\mu}$. Fix $\epsilon>0$ and let
  \begin{equation*}
    s(x):=\sum_{i=1}^kc_i\mathbf{1}_{E_i}(x),\quad x \in \R^n,
  \end{equation*}
  be a simple step function with each $c_i\in\R^m$, $\mu(E_i)<\infty$, and
  \begin{equation}
    N_{\phi,\mu}(f-s)\le \frac \epsilon 2. \label{eq: A1}
  \end{equation}
  Note that every Borel measure on $\R^n$ is regular, see e.g., \cite[Theorem~7.1.7]{Bogachev2007}. By the inner-regularity of $\mu$, for each $i$, we can choose a compact set $K_i\subset E_i$ with
  \begin{equation*}
    \mu(E_i\setminus K_i)<\frac \delta 2,
  \end{equation*}
  where $\delta>0$ will be chosen sufficiently small later. Next, by outer-regularity applied to the compact set $K_i$, we choose an open set $O_i\supset K_i$ such that
  \begin{equation*}
    \mu(O_i\setminus K_i)<\frac{\delta}{2}.
  \end{equation*}
  Since $K_i$ is compact, it is bounded. Hence, choosing $R_i>0$ sufficiently large such that $K_i\subset B(0,R_i):=\{x\in\R^n:\|x\|<R_i\}$, and setting $U_i:=O_i\cap B(0,R_i)$, we obtain a bounded open set $U_i\supset K_i$. Moreover, $U_i\setminus K_i\subset O_i\setminus K_i,$ and therefore
  \begin{equation*}
    \mu(U_i\setminus K_i)<\frac{\delta}{2}.
  \end{equation*}
  Then, by Urysohn's Lemma \cite[Lemma~4.32]{Folland1999}, there is a continuous function $u_i\colon\R^n\to[0,1]$ with
  \begin{equation*}
    u_i(x)=1\text{ for } x\in K_i
    \quad \text{and}\quad
    u_i(x)=0\text{ for }x\notin U_i.
  \end{equation*}
  
  Since $U_i$ is bounded, we have $\operatorname{supp}(u_i)\subset \overline{U_i}$,  and hence $u_i\in C_c(\R^n)$. Therefore, the finite sum
  \begin{equation*}
    h(x):=\sum_{i=1}^k c_i u_i(x), \qquad x\in\R^n,
  \end{equation*}
  defines a continuous function with compact support. More precisely,
  \begin{equation*}
    \operatorname{supp}(h)
    \subset \bigcup_{i=1}^k \operatorname{supp}(u_i)
    \subset \bigcup_{i=1}^k \overline{U_i}.
  \end{equation*}
  For each $i=1,\ldots,k$, set $A_i:=(E_i\setminus K_i)\cup (U_i\setminus K_i)$. Then
  \begin{equation*}
    |\mathbf 1_{E_i}(x)-u_i(x)|\le \mathbf 1_{A_i}(x),
    \qquad x\in\R^n.
  \end{equation*}
  Indeed, on $K_i$ both functions agree, while outside $E_i\cup U_i$ both functions vanish. Hence a possible error can only occur on $E_i\setminus K_i$ or on $U_i\setminus K_i$. Moreover, by the inner- and outer-regularity of $\mu$,
  \begin{equation*}
    \mu(A_i) \le \mu(E_i\setminus K_i)+\mu(U_i\setminus K_i) <\frac{\delta}{2}+\frac{\delta}{2} =\delta,
  \end{equation*}
  Therefore, $h$ satisfies
  \begin{align*}
    \|s(x)-h(x)\|&=\Bigl\|\sum_{i=1}^{k}c_i\mathbf 1_{E_i}(x)-\sum_{i=1}^kc_iu_i(x)\Bigr\|\\
    &\le \sum_{i=1}^{k}\|c_i\||\mathbf 1_{E_i}(x)-u_i(x)|\\
    &\le \max_{i=1,\ldots,k}\|c_i\|\sum_{i=1}^{k}\mathbf 1_{A_i}(x).
  \end{align*}
  Hence, for any $\lambda>0$ using the monotonicity of $\phi$, we have
  \begin{align*}
    \int_{\R^n}\phi\Bigl(\frac{\|s-h\|}{\lambda}\Bigr)\dd\mu&\le \int_{\R^n}\phi\Bigl(\frac{\max_{i=1,\ldots,k}\|c_i\|\sum_{i=1}^{k}\mathbf 1_{A_i}}{\lambda}\Bigr)\dd\mu\\
    &\le \int_{\bigcup_{i=1}^{k} A_i}\phi\Bigl(\frac{k\max_{i=1,\ldots,k}\|c_i\|}{\lambda}\Bigr)\dd\mu\\
    &=\mu\Bigl(\bigcup_{i=1}^{k} A_i\Bigr)\phi\Bigl(\frac{k\max_{i=1,\ldots,k}\|c_i\|}{\lambda}\Bigr)\\
    &\le \sum_{i=1}^{k}\mu(A_i)\phi\Bigl(\frac{k\max_{i=1,\ldots,k}\|c_i\|}{\lambda}\Bigr)\\
    &\le k\delta \phi\Bigl(\frac{k\max_{i=1,\ldots,k}\|c_i\|}{\lambda}\Bigr).
  \end{align*}
  Now taking
  \begin{equation*}
    \delta\le \bigg(k\phi\bigg(k\frac{\max_i \|c_i\|}{\lambda}\bigg)\bigg)^{-1},
  \end{equation*}
  yields
  \begin{equation*}
    \int_{\R^n}\phi\Bigl(\frac{\|s-h\|}{\lambda}\Bigr)\dd\mu\le 1,
  \end{equation*}
  which by definition of the gauge norm yields that $N_{\phi,\mu}(s-h)\le \lambda$. Setting $\lambda:=\frac \epsilon 2$, we have
  \begin{equation}\label{eq: A2}
    N_{\phi,\mu}(s-h)\le \frac \epsilon 2.
  \end{equation}
  Putting \eqref{eq: A1} and \eqref{eq: A2} together, we have
  \begin{equation*}
    N_{\phi,\mu}(f-h)\le N_{\phi,\mu}(f-s)+N_{\phi,\mu}(s-h)\le \frac \epsilon 2+\frac \epsilon 2=\epsilon,
  \end{equation*}
  which proves the claim.
\end{proof}

Lastly, let us recall the De la Vall{\'e}e Poussin theorem in the context of Orlicz spaces, see e.g., \cite[Theorem~1.1]{Barcenas2009}.

\begin{theorem}[De la Vall{\'e}e Poussin's Theorem]\label{thm: vallee}
  Let $(\Omega,\Sigma,\mu)$ be a finite measure\\ space. A set $A\subset L^1(\mu;\R^m)$ is uniformly integrable if and only if there is an $N$-function $\phi$ such that the set~$A$ is bounded in the Orlicz space $L^\phi(\mu;\R^m)$ with respect to the gauge norm $N_{\phi,\mu}$.
\end{theorem}

\bibliography{quellen}{}
\bibliographystyle{amsalpha}

\end{document}